\documentclass{article}

\usepackage{arxiv}

\usepackage[round]{natbib}

\usepackage{tikz} 
\usepackage{amsmath,mathpazo,mathtools,bm, amsthm}
\usepackage{dsfont} 
\usepackage{pgfplots}

\usepackage{subcaption}
\usepackage{algorithm}
\usepackage[noend]{algorithmic}

\usepackage[inline]{enumitem}

\usepackage{hyperref}

\newenvironment{sproof}{%
	\proof}{\endproof}

\newcommand{\indep}{\rotatebox[origin=c]{90}{$\models$}}
\newcommand{\nindep}{\not\!\!\rotatebox[origin=c]{90}{$\models$}}

\newtheorem{definition}{Definition}
\newtheorem{proposition}{Proposition}
\newtheorem{theorem}{Theorem}
\newtheorem*{problem}{Problem}
\newtheorem{assumption}{Assumption}

\definecolor{ugagray}{RGB}{211, 211, 211} 
\definecolor{ugablue}{RGB}{12, 35, 68} 
\definecolor{ugaorange}{RGB}{235, 91, 10} 
\definecolor{easyblue}{RGB}{3, 40, 89} 
\definecolor{easycyan}{RGB}{1, 111, 148} 
\definecolor{easyorange}{RGB}{254, 106, 44} 
\definecolor{easyred}{RGB}{211, 93, 110} 
\definecolor{easyyellow}{RGB}{249, 212, 157} 
\definecolor{easygreen}{RGB}{121, 191, 101} 

\begin{document}

%

%

%
%
%

\title{Root Cause Identification for Collective Anomalies in Time Series given an Acyclic Summary Causal Graph with Loops}


\author{\href{https://orcid.org/0000-0003-3571-3636}{\includegraphics[scale=0.06]{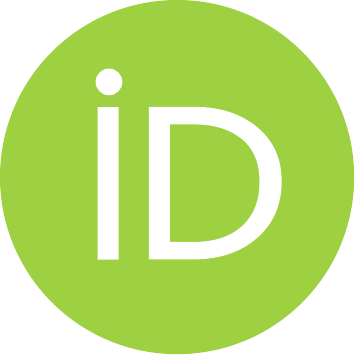}\hspace{1mm}Charles K. Assaad} \\ EasyVista 
	\And 
	Imad Ez-zejjari \\ EasyVista
	\And
	\href{https://orcid.org/0000-0003-4695-5059}{\includegraphics[scale=0.06]{orcid.pdf}\hspace{1mm}Lei Zan} \\ EasyVista \\ Univ Grenoble Alpes, \\CNRS, Grenoble INP, LIG
}

\date{}

\maketitle

\begin{abstract}
	This paper presents an approach for identifying the root causes of collective anomalies given observational time series and an acyclic summary causal graph which depicts an abstraction of causal relations present in a dynamic system at its normal regime. The paper first shows how the problem of root cause identification can be divided into many independent subproblems by grouping related anomalies using \textit{d-separation}. Further, it shows how, under this setting, some root causes can be found  directly from the graph and from the time of appearance of anomalies. Finally, it shows, how the rest of the root causes can be found by comparing direct effects in the normal and in the anomalous regime. To this end, an adjustment set for identifying direct effects is introduced. Extensive experiments conducted on both simulated and real-world datasets  demonstrate the effectiveness of the proposed method.
\end{abstract}

\section{INTRODUCTION}
\label{sec:intro}

The need for high availability of information systems requires efficient monitoring tools and a new generation of AIOps software to automate the identification of actionable root causes of anomalies in an IT monitoring system that can be used to eliminate the anomalies. 
In recent years, many approaches have been developed to identify the root causes of anomalies in multivariate time series. The most common direction that is explored considers discovering the causal graph \citep{Pearl_2000, Spirtes_2000} that represents the anomalous regime of the dynamic system \citep{Wang_2018, Meng_2020} using observational time series. 
Causal discovery methods are known to rely on strong assumptions which imply that they necessitate validation by an expert especially when there is no guarantee that these assumptions are satisfied.
In addition, causal discovery methods usually need large sample sizes \citep{Malinsky_2018, Glymour_2019, Assaad_2022_a}.
However, in many domains, the size of anomalous data depends on the sampling rate of the system. 
Thus systems with low sampling rates will collect a small size of anomalous data compared to systems with high sampling rates.
In consequence, sometimes, it is difficult to have enough data for causal discovery methods. 
And even if a sufficient amount of data is collected, the process of validation of the causal graph by an expert is time consuming and can delay the elimination of anomalies. 

To tackle this issue, we follow a different approach for root cause analysis which consists of discovering and reasoning about the summary causal graphs \citep{Assaad_2022_a} which depicts an abstraction of causal relations present in a dynamic system at its normal regime. Usually, the size of data collected in the normal regime is significantly greater than the size of data collected in the anomalous regime since anomalies are supposed to be rare. In addition, system experts can have sufficient time to validate the graph long before the appearance of anomalies. 
Note that in this work, we do not adress the problem of causal discovery of the summary causal graph of the normal regime and we assume that the graph is already learned and validated by a system expert.

This paper presents a new method for root cause identification, which we call EasyRCA, which consists of using a summary causal graph of the normal regime in order to divide the problem of root cause identification into many independent subproblems by grouping related anomalies using \textit{d-separation}. Then for each group, EasyRCA finds the root causes either directly from the graph and the time of appearance of anomalies or by comparing direct effects in the normal and anomalous regime. To this end, an adjustment set for identifying direct effects is introduced.

The remainder of the paper is organized as follows:
Section~\ref{sec:setup} introduces some terminology and formalizes the problem.
Section~\ref{sec:RW} describes related work. 
Section~\ref{sec:rca} presents our method EasyRCA which is evaluated on simulated and real datasets in Section~\ref{sec:exp}. Finally, Section~\ref{sec:conc} concludes the paper.

\section{Problem setup}
\label{sec:setup}

In this section, we first introduce some terminology, tools, and assumptions which are standard for the major part. Then, we formalize the problem we are going to solve.

Suppose that a dynamic system can be represented by a structural causal model (SCM) \citep{Pearl_2000} in which each point in a time series is given by a function (so-called causal mechanism) of its parents and an unobserved noise:
\begin{equation}
\label{eq:SCM}
Y_t := f^y_t(Parents(Y_t), \xi^y_t)
\end{equation}
where the noise variables are jointly statistically independent so that there are no hidden confounding , i.e., causal sufficiency is satisfied \citep{Spirtes_2000}. The qualitative causal relations induced by such SCM can be represented by a causal graph in which, under the causal Markov condition \citep{Spirtes_2000}  each vertex is independent of all other vertices given its parents, except for its descendants. In dynamic systems, these causal graphs are referred to as full-time causal graphs. An example of such a graph is presented in Figure~\ref{fig:full-graph}.
The main difficulty in working with this type of graph is that it is infinite and so in practice, inferring it is unfeasible. 
However, it is very likely that causal relations between two time series will hold throughout time as such relations are generally associated with underlying physical
processes. Thus we can assume consistency throughout time.
\begin{assumption}[Consistency throughout time, \citep{Assaad_2022_a}]
	A full time causal graph is said to be \emph{consistent throughout time} if all the causal relationships remain constant in direction throughout time.
	\label{def:consistence-time}
\end{assumption}
Under this assumption, the full time causal graph can be contracted to give a finite graph which is called a window  causal graph. It is a representation of the causal relations through a time window, the size of which depends on the maximum lag between a cause and an effect in the full time causal graph. An example of a window causal graph is given in Figure\ref{fig:window-graph}.
This said, it is usually difficult for an expert to validate, analyze let alone provide a window causal graph because it is difficult to determine the temporal lag between a cause and an effect. Thus, experts usually rely on the so-called summary causal graph which is a compact version of the window causal graph that represents the causal relations between time series without giving any information about the temporal lags of these relations.
In this work, we assume that the summary causal graph is acyclic but loops are allowed to represent temporal dependencies within the same time series. 
An example of such graph is given in Figure~\ref{fig:summary-graph} and formally it is defined as follows:
\begin{definition}[Acyclic summary causal graph with loops]
	\label{defCausalGraph}
	Consider $\mathcal{G} = (\mathcal{V}, \mathcal{E})$ is a \emph{summary causal graph}. The set of vertices in that graph consists of the set of time series. The arcs $\mathcal{E}$ of the graph are defined as follows: $\forall X, Y \in \mathcal{V}$, $X$ causes $Y$ if and only if there exists some time lag $\gamma$ such that $X_{t-\gamma}$ causes $Y_{t}$ such that $\gamma \ge 0$ for $X \ne Y$ and $\gamma>0$ for $X=Y$. If $\mathcal{G}= (\mathcal{V}, \mathcal{E})$ has no directed cycles other than the edges going from one vertex to itself, then $\mathcal{G} = (\mathcal{V}, \mathcal{E})$ is said to be an \emph{acyclic summary causal graph with loops} (ASCGL).
\end{definition}
In this work, we suppose that the ASCGL is given either using experts knowledge or learned directly from observational time series \citep{Peters_2013, Assaad_2021, Assaad_2022_c} or by first discovering a window causal graph\footnote{In our framework, if one decides using a causal discovery algorithm that learns a window causal graph, then one needs to carefully incorporate to these algorithms the constraint that the summary causal graphs should be acyclic.} \citep{Runge_2019, Runge_2020} and then deduce the ASCGL from it.
The correctness of such learned graphs usually rests on untestable assumptions and depends on the quality of the data so it is important to validate it by an expert or to simplify the problem of causal discovery by providing some background knowledge.

\begin{figure*}[!ht]
	\centering
	\begin{subfigure}{0.48\textwidth}
		\centering
		\begin{tikzpicture}[{black, circle, draw, inner sep=0}]
		\tikzset{nodes={draw,rounded corners},minimum height=0.7cm,minimum width=0.7cm}
		\tikzset{latent/.append style={fill=gray!30}}
		
		\node (X-2) at (0,-1) {$Z_{t-2}$} ;
		\node (X-1) at (3,-1) {$Z_{t-1}$};
		\node (X) at (6,-1) {$Z_{t}$};
		\node (Z-2) at (0,0) {$X_{t-2}$} ;
		\node (Z-1) at (3,0) {$X_{t-1}$};
		\node (Z) at (6,0) {$X_{t}$};
		\node (Y-2) at (0,1) {$W_{t-2}$} ;
		\node (Y-1) at (3,1) {$W_{t-1}$};
		\node (Y) at (6,1) {$W_{t}$};
		\node (W-2) at (0,2) {$Y_{t-2}$} ;
		\node (W-1) at (3,2) {$Y_{t-1}$};
		\node (W) at (6,2) {$Y_{t}$};
		
		\draw[->,>=latex] (Z-2) -- (Z-1);
		\draw[->,>=latex] (Z-1) -- (Z);
		\draw[->,>=latex] (Y-2) -- (Y-1);
		\draw[->,>=latex] (Y-1) -- (Y);
		\draw[->,>=latex] (X-2) -- (X-1);
		\draw[->,>=latex] (X-1) -- (X);
		\draw[->,>=latex] (W-2) -- (W-1);
		\draw[->,>=latex] (W-1) -- (W);
		
		\draw[->,>=latex] (X-2) -- (Z-1);
		\draw[->,>=latex] (X-1) -- (Z);
		\draw[->,>=latex] (X-2) -- (Y-1);
		\draw[->,>=latex] (X-1) -- (Y);

		\draw[->,>=latex] (Z-1) -- (Z);
		\draw[->,>=latex] (Y-1) -- (Y);
		
		\draw[->,>=latex] (Y) to [out=45,in=-45, looseness=1] (W);
		\draw[->,>=latex] (Z) to [out=45,in=-25, looseness=1] (W);		
		\draw[->,>=latex] (Y-1) to [out=45,in=-45, looseness=1] (W-1);
		\draw[->,>=latex] (Z-1) to [out=45,in=-25, looseness=1] (W-1);	
		\draw[->,>=latex] (Y-2) to [out=45,in=-45, looseness=1] (W-2);
		\draw[->,>=latex] (Z-2) to [out=45,in=-25, looseness=1] (W-2);
		
		\coordinate[left of=X-2] (d1);
		\draw [dashed,>=latex] (X-2) to[left] (d1);
		\coordinate[left of=Z-2] (d1);
		\draw [dashed,>=latex] (Z-2) to[left] (d1);
		\coordinate[left of=Y-2] (d1);
		\draw [dashed,>=latex] (Y-2) to[left] (d1);		
		\coordinate[left of=W-2] (d1);
		\draw [dashed,>=latex] (W-2) to[left] (d1);
		
		\coordinate[right of=X] (d1);
		\draw [dashed,>=latex] (X) to[right] (d1);
		\coordinate[right of=Z] (d1);
		\draw [dashed,>=latex] (Z) to[right] (d1);
		\coordinate[right of=Y] (d1);
		\draw [dashed,>=latex] (Y) to[right] (d1);
		\coordinate[right of=W] (d1);
		\draw [dashed,>=latex] (W) to[right] (d1);		
		\end{tikzpicture}
		\caption{Full time causal graph}
		\label{fig:full-graph}
	\end{subfigure}%
	\hfill
	\begin{subfigure}{.3\textwidth}
		\centering
		\begin{tikzpicture}[{black, circle, draw, inner sep=0}]
		\tikzset{nodes={draw,rounded corners},minimum height=0.7cm,minimum width=0.7cm}
		\tikzset{latent/.append style={fill=gray!30}}
		
		\node (X-1) at (3,-1) {$Z_{t-1}$};
		\node (X) at (6,-1) {$Z_{t}$};
		\node (Z-1) at (3,0) {$X_{t-1}$};
		\node (Z) at (6,0) {$X_{t}$};
		\node (Y-1) at (3,1) {$W_{t-1}$};
		\node (Y) at (6,1) {$W_{t}$};
		\node (W-1) at (3,2) {$Y_{t-1}$};
		\node (W) at (6,2) {$Y_{t}$};
		
		\draw[->,>=latex] (Y) to [out=45,in=-45, looseness=1] (W);
		\draw[->,>=latex] (Z) to [out=45,in=-25, looseness=1] (W);		
		\draw[->,>=latex] (Y-1) to [out=45,in=-45, looseness=1] (W-1);
		\draw[->,>=latex] (Z-1) to [out=45,in=-25, looseness=1] (W-1);	
		
		\draw[->,>=latex] (Z-1) -- (Z);
		\draw[->,>=latex] (Y-1) -- (Y);
		\draw[->,>=latex] (X-1) -- (X);
		\draw[->,>=latex] (W-1) -- (W);
		\draw[->,>=latex] (X-1) -- (Z);
		\draw[->,>=latex] (X-1) -- (Y);
		
		\draw[->,>=latex] (Z-1) -- (Z);
		\draw[->,>=latex] (Y-1) -- (Y);
		\end{tikzpicture}
		\caption{Window  causal graph}
		\label{fig:window-graph}
	\end{subfigure}\hfill
	\begin{subfigure}{.22\textwidth}
		\centering
		\begin{tikzpicture}[{black, circle, draw, inner sep=0}]
		\tikzset{nodes={draw,rounded corners},minimum height=0.7cm,minimum width=0.7cm}
		\tikzset{latent/.append style={fill=gray!30}}
		
		\node (X) at (0,-0.4) {$X$} ;
		\node (Z) at (1,1) {$Z$};
		\node (Y) at (2,-0.4) {$W$};
		\node (W) at (1,-1.8) {$Y$};
		\draw[->,>=latex] (Y) to [out=0,in=45, looseness=2] (Y);
		\draw[->,>=latex] (Z) to [out=0,in=45, looseness=2] (Z);
		\draw[->,>=latex] (X) to [out=180,in=135, looseness=2] (X);
		\draw[->,>=latex] (W) to [out=0,in=-45, looseness=2] (W);
		
		\draw[->,>=latex] ( Z) -- (X);
		\draw[->,>=latex] (Z) -- (Y);
		
		\draw[<-,>=latex] (W) -- (X);
		\draw[<-,>=latex] (W) -- (Y);
		\end{tikzpicture}
		\caption{Summary causal graph}
		\label{fig:summary-graph}
	\end{subfigure}
	\hspace{0.8cm}
	\caption[Different causal graphs that one can infer from three time series.]{Different causal graphs that one can infer from three time series: full time causal graph \eqref{fig:full-graph}, window  causal graph \eqref{fig:window-graph} and summary causal graph \eqref{fig:summary-graph}. Note that the first one gives  more information but cannot be inferred in practice, the second one is a schematic viewpoint of the full behavior, whereas the last one is an abstraction and can be deduced from the window causal graph.}
	\label{fig:graphs}
\end{figure*}
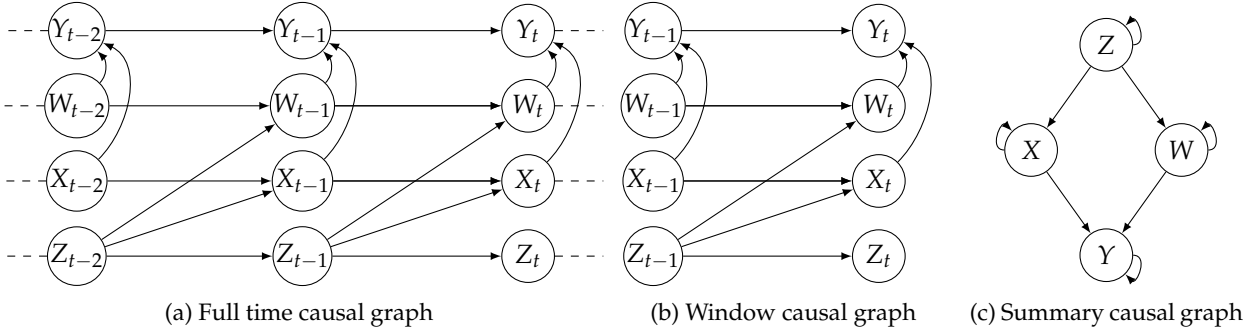

Now, we turn our focus on anomalies. In this work, we assume that anomalies are collective.
\begin{definition}[Collective anomaly, \cite{Chandola_2009}]
	\label{def:collective_anomaly}
	In time series, a collective anomaly is a sequence of data instances that is anomalous with respect to the entire time series.
\end{definition}
Point anomalies are disregarded because we are interested in finding actionable root causes that can eliminate anomalies. If an anomaly appears for one time instant and then disappears this means it was eliminated on its own and do not require any action to resolve it. 
However, we also consider that those collective anomalies have a limited size since information systems are expected to be highly available (without anomalies). 
In addition, we assume that each anomaly can be eliminated by removing the intervention that caused it directly or via a causal path. This is given by the following assumption:

\begin{assumption}
	\label{ass:anomaly_propagation}
	All anomalies are propagated from an external intervention through the structural causal model.
\end{assumption}

Given the definition of ASCGL and collective anomalies, we define root causes\footnote{The "root causes" are relative to the set of observed time series.} as follows:
\begin{definition}[Root causes]
	\label{def:rc}
	Given an ASCGL and a set of anomalous vertices $\mathcal{A}$, the set of root causes $\mathcal{C}$ of $\mathcal{A}$ is a set of vertices that were affected by an external intervention which led to marginal distribution change in $\mathcal{A}$.
\end{definition}
In the literature, there exists two sorts of interventions \citep{Eberhardt_2007} and both are crucial to root cause analysis.
The first is known as parametric intervention and it is defined as follows:
\begin{definition}[Parametric intervention]
	Consider an ASCGL $\mathcal{G}=(\mathcal{V}, \mathcal{E})$.
	An intervention  on a vertex $Y\in\mathcal{V}$ is parametric if the causal mechanism before intervention is different than after the intervention but $Parents(Y, \mathcal{G})$ remains unchanged.
\end{definition}
The second type of intervention is known as structural intervention\footnote{Our definition of structural intervention is less restrictive than the classical definition which states that the intervention alone completely determines the probability distribution of the variable that underwent the intervention, i.e., this variable becomes independent of all of its parents.} 
and it is defined as follows:
\begin{definition}[Structural intervention]
	Consider an ASCGL $\mathcal{G}=(\mathcal{V}, \mathcal{E})$.
	An intervention  on a vertex $Y\in\mathcal{V}$ is structural if $\exists X\in\mathcal{V}$ such that $X\in Parents(Y, \mathcal{G})$ before the intervention and $X\not\in Parents(Y, \mathcal{G})$ after the intervention.
\end{definition}

Structural interventions can be regarded as a special case of parametric interventions. But we distinguish between them because an expert might want to differentiate between interventions that provoke a disruption in the system from the ones that do not.
As it will be shown in Section~\ref{sec:rca}, to find these two types of interventions we will estimate the direct effects in the normal and anomalous regime. In consequence, we assume that causal mechanisms are fixed throughout time within the same regime (e.g., in Equation~\ref{eq:SCM}, $f^y(.)$  is fixed for all $t$) and we assume the minimality condition\footnote{The minimality condition is usually assumed by causal discovery methods either directly or by assuming a stronger assumption called faithfulness which implies the minimality condition \citep{Glymour_2019}.} \citep{Spirtes_2000} which implies that adjacent vertices in the ASCGL are statistically dependent in the normal regime.
Finally, to simplify the problem we assume linear SCMs.

Now that we have introduced the needed tools and assumptions, the problem we are trying to solve is formalized as follows:
\begin{problem}
	Given an ASCGL $\mathcal{G}=(\mathcal{V}, \mathcal{E})$, a set of anomalous vertices $\mathcal{A} \subset \mathcal{V}$, the distribution of the time series in the normal regime $\mathcal{{N}}$ and in the anomalous regime $\mathcal{\bar{N}}$, and the maximal lag between a cause and an effect $\gamma_{max}$, we want to find the smallest set of root causes $\mathcal{C}$ of $\mathcal{A}$.
\end{problem}

\section{RELATED WORKS}
\label{sec:RW}

Recently, there has been an increase in the popularity of automating the process of root cause analysis. Among the most popular unsupervised methods that deal with time series is CloudRanger \citep{Wang_2018} which is decomposed into two steps. First, it discovers the summary causal graph between anomalous time series using the PC algorithm \citep{Spirtes_2000} which was introduced for non-temporal data.
Then, identifies root causes through random walk based on a transition matrix computed using the correlation between time series.  
As one might expect, the main limitation of this method is that it uses a non-temporal algorithm that does not take into account temporal lags that might exist between two time series. In addition, correlation does not necessarily represent the causal effect of one variable on  another.
To fix these issues, \cite{Meng_2020}  proposed a similar method, called MicroCause, where the causal discovery is done using the PCMCI\footnote{There exists two versions of PCMCI, one that allows for instantaneous relations \citep{Runge_2020} and one that does not \citep{Runge_2019}. In the experimentation section, we use the version that allows instantaneous relations.} \citep{Runge_2019, Runge_2020} algorithm, an extension of the PC algorithm for time series which infers a window causal graph. Then, MicroCause deduces from the inferred window causal graph a summary causal graph. Furthermore, to compute the transition matrix for the random walk,  MicroCause estimates the partial correlation between each causally related time series given their parents in the graph. Conditioning on the parents is a sufficient condition to eliminate all spurious correlations when there are no hidden common causes.
Note that PC and PCMCI use conditional independencies to infer the causal graph and such methods are known to be correct when the faithfulness condition is satisfied. 
Closer to our proposal, \cite{Budhathoki_2021} introduced a formal method, that we will denote as WhyMDC, to detect the root cause of a change in a marginal distribution from non temporal data. WhyMDC considers that a directed acyclic causal graph is given and as far as we know, it is the first method to identify root causes by searching for changes in causal mechanisms.

There exist other root cause identification methods which are beyond the scope of this paper. For example, \cite{Budhathoki_2022} proposed a root cause analysis framework to detect the root cause of a point anomaly using non-temporal structural causal models 
and \cite{Zhang_2022} proposed a supervised learning approach to find root causes.

\section{ROOT CAUSE IDENTIFICATION USING ASCGLs}
\label{sec:rca}

\subsection{Grouping related anomalies}

We first give an extension of the concept of d-seperation to ASCGL and then show how it can be used to divide the root cause identification problem into many independent subproblems. 

A path is said to be \emph{blocked} by a set of vertices $\mathcal{Z} \in \mathcal{V}$ if it contains a chain $X \rightarrow W \rightarrow Y$ or a fork $X \leftarrow W \rightarrow Y$ and $W \in \mathcal{Z}$, or it contains a collider $X \rightarrow W \leftarrow Y$ such that no descendant of $W$ is in $\mathcal{Z}$.
A path is said to be \emph{active} if it is not blocked. Using blocked paths, the notion of d-separation is defined as follows:

\begin{definition}[d-separation, \cite{Pearl_2000}]
	\label{def:d_sep}
	Given a DAG $\mathcal{G}=(\mathcal{V}, \mathcal{E})$ and disjoint sets $\mathcal{X}, \mathcal{Y}, \mathcal{Z} \subseteq \mathcal{V}$, $\mathcal{X}$ and $\mathcal{Y}$ are \emph{d-separated} by $\mathcal{Z}$ if every path between a vertex in $\mathcal{X}$ and a vertex in $\mathcal{Y}$ is blocked by $\mathcal{Z}$.  We write d-separated as $\mathcal{X} \indep_{\mathcal{G}} \mathcal{Y} \mid \mathcal{Z}$.
\end{definition}

Note that d-separation was introduced for directed acyclic graph (DAG), so it is directly applicable for full time graphs and window causal graphs but not for summary causal graphs. However, it turned out that the extention to ASCGL is simple. 
If there are no loops, d-separation in an ASCGL is equivalent to the one in Definition~\ref{def:d_sep}. For example, in Figure~\ref{fig:summary-graph}, if we omit the loops then it is obvious that $X\indep_{\mathcal{G}} W \mid Z$.  
If there are loops, at first glance, Definition~\ref{def:d_sep} might seem to fail. However, The following proposition shows how Definition~\ref{def:d_sep} can still be used for ASCGL.

\begin{proposition}
	\label{def:d_sep_summary_causal_graph}
	Given an ASCGL $\mathcal{G}=(\mathcal{V}, \mathcal{E})$ and disjoint sets $\mathcal{X}, \mathcal{Y}, \mathcal{Z} \subseteq \mathcal{V}$, $\mathcal{X}\indep_{\mathcal{G}}\mathcal{Y}\mid \mathcal{Z}$ if $\mathcal{Z}= Parents(\mathcal{X}, \mathcal{G})\cup Parents(\mathcal{Y}, \mathcal{G})$ and $\mathcal{X}\indep_{\mathcal{G}'} \mathcal{Y}\mid \mathcal{Z}'$ such that $\mathcal{G}'$ is a DAG identical to $\mathcal{G}$ but loops are omitted and $\mathcal{Z}'=Parents(\mathcal{X}, \mathcal{G}')\cup Parents(\mathcal{Y}, \mathcal{G}')$. 
\end{proposition}
\begin{proof}
	Consider a DAG $\mathcal{G}'$ without loops. Suppose $\mathcal{X}, \mathcal{Y} \subseteq \mathcal{V}'$, $\mathcal{Z}^x= Parents(\mathcal{X},\mathcal{G}')$ and $\mathcal{Z}^y= Parents(\mathcal{Y}, \mathcal{G}')$ such that $\mathcal{X}\indep_{\mathcal{G}'}\mathcal{Y}\mid \mathcal{Z}^x\cup \mathcal{Z}^y$. Now consider an identifical graph $\mathcal{G}$ but with loops. $\forall X_{t-\gamma_{xy}}, Y_t$ such that $\gamma_{xy}\in \mathds{N}$, the set $\mathcal{Z}^{x,t}=\mathcal{Z}^x_{t-\gamma_{xy}}\cup \cdots \cup \mathcal{Z}^x_{t-\gamma_{xy}-\gamma_{max}}$ contains all parents of $X_{t-\gamma}$ in $\mathcal{Z}^x$ and none of its decedants, thus $\mathcal{Z}^{x,t}$ blocks all active paths going into $X_{t-\gamma_{xy}}$ that does not pass by the past $X_{t-\gamma_{xy}}$. It follows that $\mathcal{Z}^{x,t}\cup X_{t-\gamma_{xy}-1}, \cdots, X_{t-\gamma_{xy}-\gamma_{max}}$  blocks all active paths going into $X_{t-\gamma_{xy}}$. Similarly, $\mathcal{Z}^y_{t}\cup \cdots \cup \mathcal{Z}^y_{t-\gamma_{max}}\cup Y_{t-1}, \cdots, Y_{t-\gamma_{max}}$ blocks all active paths going into $Y_{t}$. Therefore, $\mathcal{X}\indep_{\mathcal{G}}\mathcal{Y}\mid \mathcal{Z}^x\cup \mathcal{Z}^y$ in $\mathcal{G}$.
\end{proof}
Note that we focused on parents and excluded ancestors to avoid separation sets of infinite size.
For example, in Figure~\ref{fig:summary-graph}, we can explain why $X\nindep_{\mathcal{G}} W \mid Z$ by looking at the compatible full-time causal graph in Figure~\ref{fig:full-graph} where $X_t\nindep_{\mathcal{G}} W_t\mid Z_{t-1}, X_{t-1}, W_{t-1}$.


Assumption~\ref{ass:anomaly_propagation} and Definition~\ref{def:rc} imply that root causes of anomalous vertices $\mathcal{A}$ is a subset of $\mathcal{A}$.
So in the following, we consider that $\mathcal{A}= \mathcal{C}\cup \mathcal{\bar{C}}$. Such that $\mathcal{C}$ represents root causes and $\mathcal{\bar{C}}$ represents non root causes.

\begin{proposition}
	Given an ASCGL $\mathcal{G}=(\mathcal{V}, \mathcal{E})$ and anomalous vertices $\mathcal{A}\subseteq \mathcal{V}$ such that $\mathcal{A}=\mathcal{C}\cup \mathcal{\bar{C}}$, $\forall \mathcal{S} \subseteq \mathcal{V}\backslash\mathcal{A}$, $\mathcal{C}\nindep_{\mathcal{G}} \mathcal{\bar{C}} \mid \mathcal{S}$.
\end{proposition}

\begin{proof}
	Consider anomalous vertices $\mathcal{A}=\mathcal{C}\cup \mathcal{\bar{C}}$, such that $\mathcal{C}$ is the set of root causes of $\mathcal{\bar{C}}$. If $\exists \mathcal{S}\subseteq \mathcal{V}\backslash\mathcal{A}$ such that $\mathcal{C}\indep_{\mathcal{G}} \mathcal{\bar{C}} \mid \mathcal{S}$ then by definition of d-separation all paths between $\mathcal{C}$ and $\mathcal{\bar{C}}$ are blocked given $\mathcal{S}$ which means all directed paths from $\mathcal{C}$ to $\mathcal{\bar{C}}$ are blocked given $\mathcal{S}$. In this case, $\forall X \in \mathcal{C}, \forall Y \in \mathcal{\bar{C}}$, there exists no directed path $\pi$ from $X$ to $Y$ such that each vertex on $\pi$  belongs to $\mathcal{A}$. It follows that $\mathcal{\bar{C}}$ is not propagated from $\mathcal{C}$ which contradicts Assumption~\ref{ass:anomaly_propagation}. Hence it must be the case that $\forall \mathcal{S} \subseteq \mathcal{V}\backslash\mathcal{A}$, $\mathcal{C}\nindep_{\mathcal{G}} \mathcal{\bar{C}} \mid \mathcal{S}$.
\end{proof}

\begin{definition}[Linked anomalous graph]
	\label{def:LAG}	
	Given an ASCGL $\mathcal{G}=(\mathcal{V}, \mathcal{E})$ and a set of anomalous vertices $\mathcal{A} \subset \mathcal{V}$. $\mathcal{L}=\{\mathcal{L}^1, \cdots, \mathcal{L}^m\}$ is a set of linked anomalous graphs if $\forall i \in\{1,\cdots, m\}$ $\mathcal{L}^i=(\mathcal{A}^i, \mathcal{E}^i)$ is a subgraph of $\mathcal{G}$ such that $\mathcal{A}^i \subset \mathcal{A}$ and there exists a set of vertices $\mathcal{S} \subset \mathcal{V}\backslash\mathcal{A}$ such that $\mathcal{A}^i\indep_{\mathcal{G}} \mathcal{A}\backslash\mathcal{A}^i \mid \mathcal{S}$. 
\end{definition}

For example, consider the ASCGL in Figure~\ref{fig:linked_anomalous_graphs}: $B,C,D,W,X,Y,Z$ are anomalous vertices and $A$ is a normal vertex. Since $B,C,D \indep_{\mathcal{G}} W,X,Y,Z\mid A$ and $B,C,D$ and $W,X,Y,Z$ are respectively d-connected given $A$, then $B,C,D$ and $W,X,Y,Z$ form two linked anomalous graphs.


\begin{proposition}
	\label{prop:disjoint_set_of_link_anomalous_graphs}
	Given an ASCGL $\mathcal{G}=(\mathcal{V}, \mathcal{E})$ if the set of linked anomalous graphs is $\mathcal{L}=\{\mathcal{L}^1, \cdots, \mathcal{L}^m\}$, then $\forall i, j \in \{1, \cdots, m\}, \mathcal{L}^i\cap \mathcal{L}^j = \emptyset$.
\end{proposition}
\begin{proof}
	Consider two different linked anomalous graphs $\mathcal{L}^1=(\mathcal{A}^1, \mathcal{E}^1)$ and $\mathcal{L}^2=(\mathcal{A}^2, \mathcal{E}^2)$ such that $\mathcal{A}^1\indep_{\mathcal{G}} \mathcal{A}\backslash\mathcal{A}^1 \mid \mathcal{S}^1$ and $\mathcal{A}^2\indep_{\mathcal{G}} \mathcal{A}\backslash\mathcal{A}^2 \mid \mathcal{S}^2$ such that $\mathcal{S}^1, \mathcal{S}^2 \subseteq \mathcal{V}\backslash\mathcal{A}$. 
	It follows that if $\mathcal{L}^1\cap \mathcal{L}^2 \ne \emptyset$ then $\exists X\in\mathcal{V}$ such that $X\in \mathcal{A}^1$ and $X\in \mathcal{A}^2$. In consequence, $\not\exists \mathcal{S}\subseteq\mathcal{V}\backslash\mathcal{A}$ such that  $\mathcal{A}^1~\nindep_{\mathcal{G}} \mathcal{A}^2 \mid \mathcal{S}$. Which means according to Definition~\ref{def:LAG}, $\mathcal{A}^1$ and $\mathcal{A}^2$ belong to the same linked anomalous graph. 
\end{proof}

\begin{proposition}
	\label{prop:modularity_of_link_anomalous_graphs}
	Given an ASCGL $\mathcal{G}=(\mathcal{V}, \mathcal{E})$ if the set of linked anomalous graphs is $\mathcal{L}=\{\mathcal{L}_1, \cdots, \mathcal{L}_m\}$, then $\forall i, j \in \{1, \cdots, m\}, \mathcal{C}_i\cap \mathcal{C}_j = \emptyset$ such that $\mathcal{C}_i$ is the set of root causes of $\mathcal{L}_i$ and   $\mathcal{C}_j$ is the set of root causes of $\mathcal{L}_j$.
\end{proposition}

\begin{proof}
	This follows from Proposition~\ref{prop:disjoint_set_of_link_anomalous_graphs}.
\end{proof}

Propositions~\ref{prop:disjoint_set_of_link_anomalous_graphs}~ and~\ref{prop:modularity_of_link_anomalous_graphs} suggest that linked anomalous graphs are modular with respect to each other, which implies that
the set of root causes of each linked anomalous graph can be identified independently of the rest of the anomalies in the graph.

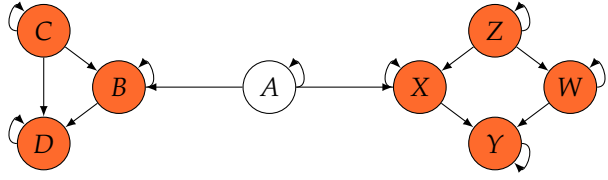
\begin{figure}
	\centering
	\begin{tikzpicture}[{black, circle, draw, inner sep=0}]
	\tikzset{nodes={draw,rounded corners},minimum height=0.7cm,minimum width=0.7cm}
	\tikzset{anomalous/.append style={fill=easyorange}}
	
	\node[anomalous] (X) at (0,0) {$X$} ;
	\node[anomalous] (Z) at (1,0.75) {$Z$};
	\node[anomalous] (Y) at (2,0) {$W$};
	\node[anomalous] (W) at (1,-0.75) {$Y$};
	
	\node (A) at (-2,0) {$A$};
	\node[anomalous] (B) at (-4,0) {$B$};
	\node[anomalous] (C) at (-5,0.75) {$C$};
	\node[anomalous] (D) at (-5,-0.75) {$D$};
	
	\draw[->,>=latex] (Y) to [out=0,in=45, looseness=2] (Y);
	\draw[->,>=latex] (Z) to [out=0,in=45, looseness=2] (Z);
	\draw[->,>=latex] (X) to [out=180,in=135, looseness=2] (X);
	\draw[->,>=latex] (W) to [out=0,in=-45, looseness=2] (W);
	
	\draw[->,>=latex] ( Z) -- (X);
	\draw[->,>=latex] (Z) -- (Y);	
	\draw[<-,>=latex] (W) -- (X);
	\draw[<-,>=latex] (W) -- (Y);
	
	\draw[->,>=latex] (A) -- (X);	
	\draw[->,>=latex] (A) -- (B);	
	\draw[->,>=latex] (C) -- (B);	
	\draw[->,>=latex] (B) -- (D);	
	\draw[->,>=latex] (C) -- (D);	
	
	\draw[->,>=latex] (A) to [out=0,in=45, looseness=2] (A);
	\draw[->,>=latex] (B) to [out=0,in=45, looseness=2] (B);
	\draw[->,>=latex] (C) to [out=180,in=135, looseness=2] (C);
	\draw[->,>=latex] (D) to [out=180,in=135, looseness=2] (D);
	
	\end{tikzpicture}
	\caption{An ASCGL with two linked anomalous graphs. White vertices represents normal vertices and orange vertices represents anomalous vertices.}
	\label{fig:linked_anomalous_graphs}
\end{figure}

Next, we will show how to detect a subset of the root causes uniquely by looking at the graph and the time of the first appearance of anomalies on each vertex.

\subsection{Identifying root causes from the graph}\label{subsec:rc_from_graph}

\begin{definition}[Sub-root vertex]
	\label{def:sub_root}
	A sub-root vertex is root vertex in a linked anomalous graph. 
\end{definition}

\begin{definition}[Time defying vertex]
	\label{def:time_defying}
	Consider a linked anomalous graph $\mathcal{L}^i=(\mathcal{A}^i,\mathcal{E}^i)$. $Y$ is a time defying vertex if and only if $\forall X \in Parents(Y, \mathcal{L}^i)$ the appearance time of the anomaly on $Y$ precedes the appearance time of the anomaly on $X$.
\end{definition}

\begin{proposition}
	\label{prop:roots_sub_root_time_defying}
	Given a linked anomalous graph $\mathcal{L}^i=(\mathcal{A}^i,\mathcal{E}^i)$, its set of sub-root vertices $\mathcal{R}^i$, and its set of time defying vertices $\mathcal{T}^i$, then $\mathcal{R}^i\cup \mathcal{T}^i \subseteq \mathcal{C}^i$ such that $\mathcal{C}^i$ is the true set of root causes in $\mathcal{L}^i$.
\end{proposition}

\begin{proof}
	Consider a linked anomalous graph $\mathcal{L}^i=(\mathcal{A}^i,\mathcal{E}^i)$.
	1) By Definition~\ref{def:sub_root}, a sub-root vertex in $\mathcal{L}^i$ does not have any anomalous parent, it follows from Assumption~\ref{ass:anomaly_propagation} and Propositions~\ref{prop:disjoint_set_of_link_anomalous_graphs},\ref{prop:modularity_of_link_anomalous_graphs} that the anomaly in this vertex cannot be propagated from other vertices, which implies that it was itself directly affected by an external intervention which means it is a root cause.
	2) Consider two anomalous vertices $X,Y \in \mathcal{A}^i$ such that $Parents(Y, \mathcal{L}^i) = \{X\}$. According to Assumption~\ref{ass:anomaly_propagation}, the anomaly in $X$ which appeared at time $t$ was propagated to $Y$ according to the lag $\gamma\ge 0$ between $X$ and $Y$ of the SCM. 
	It follows that if the appearance time of the anomaly on $Y$ is $t'$ such that $t'<t\le t+ \gamma$ then the anomaly on $Y$ was not propagated from $X$, which implies that $Y$ was itself directly affected by an external intervention which means it is a root cause.
	Using induction, suppose this is true for $Parents(Y, \mathcal{L}^i)=\{X^1, \cdots, X^p\}$.
	If $Parents(Y, \mathcal{L}^i)=\{X^1, \cdots, X^{p+1}\}$ such that $t'$ precedes the appearance time of the anomalies on $\{X^1, \cdots, X^{p}\}$ then if $t'$ does not precede the appearance time on $X^{p+1}$, then the anomaly on $Y$ could have been propagated from $X^{p+1}$, otherwise, we conclude that $Y$ was directly affected by an external intervention which means it is a root cause.
\end{proof}

Proposition~\ref{prop:roots_sub_root_time_defying} states that if a vertex is a sub-root or a time-defying vertex in a linked anomalous graph then it belongs to the set of root causes. However, this does not mean that every element in the set of root causes is necessarily a sub-root or a time-defying vertex. 
Therefore, there might be a vertex $X \in \mathcal{A}^i\backslash\{\mathcal{R}^i\cup  \mathcal{T}^i\}$, such that $X \in \mathcal{C}^i$.

\subsection{Identifying root causes from data}\label{subsec:rc_from_graph_data}

To find root causes that are neither sub-roots nor time-defying vertices, we search for changes in the causal mechanisms.
In some cases, we can find these changes, for each $X\rightarrow Y$ in a given linked anomalous graph, by estimating the total effect\footnote{The \emph{do} operator  represents an external intervention.} \citep{Pearl_2000} 
of $X$ on $Y$ in the normal regime, defined as
\begin{align}
TE^{{N}}_{X_{t-\gamma_{xy}}\rightarrow Y_t} =& E_{N}[Y_t|do(X_{t-\gamma_{xy}}=x)]  \\
&- E_{N}[Y_t|do(X_{t-\gamma_{xy}}=x')] \nonumber
\end{align}
and the total effect 
of $X$ on $Y$ in the anomalous regime, defined as
\begin{align}
TE^{\bar{N}}_{X_{t-\gamma_{xy}}\rightarrow Y_t} =& E_{\bar{N}}[Y_t|do(X_{t-\gamma_{xy}}=x)]  \\
&- E_{\bar{N}}[Y_t|do(X_{t-\gamma_{xy}}=x')],\nonumber
\end{align} 
where $E_{N}$ and $E_{\bar{N}}$ are respectivly the expectations in the normal regime and in the anomalous regime and the temporal lag $\gamma_{xy}$ is between $\bar\gamma_{xy}$ and $\gamma_{max}$ such that $\bar\gamma_{xy}$  can be found by the substracting the time of appearance of anomalies on $Y$ from the time of appearance of anomalies on $X$.
If there is no directed path from $X$ to $Y$ (other than $X\rightarrow Y$) or if directed paths between $X$ and $Y$ exist but we know that all vertices (other than $X$ and $Y$) on these paths cannot be root causes, then if $TE^{{N}}_{X_{t-\gamma_{xy}}\rightarrow Y_t} \ne TE^{\bar{N}}_{X_{t-\gamma_{xy}}\rightarrow Y_t}$ we can conclude that there is a change in the causal mechanism of $Y$ provoked by an external intervention on $Y$.

In order to estimate $TE^{N}_{X_{t-\gamma_{xy}}\rightarrow Y_t}$ and $TE^{\bar{N}}_{X_{t-\gamma_{xy}}\rightarrow Y_t}$ from observational data we need to eliminate the \emph{do} from the total effect expression.
This can be achieved, when there is no hidden common causes, by the back-door criterion \citep{Pearl_2000} which searches for an adjustment set of vertices, called back-door set, that eliminates all spurious correlations between $X$ and $Y$. However, the back-door criterion cannot directly be applied to ASCGL because of loops. 
In the following we present an adjustment set for identifying total effects\footnote{A similar result was presented in \citep{Eichler_2007} for summary causal graphs assuming there is no instantaneous relations and allowing for sets of infinite size.} from ASCGLs:

\begin{definition}[An adjustment set for total effects in an ASCGL]
	\label{def:back_door_summarygraph}
	Consider an ASCGL $\mathcal{G}=(\mathcal{V}, \mathcal{E})$, a maximal lag $\gamma_{max}$, two vertices $X$ and $Y$ such that $X\rightarrow Y$ in $\mathcal{G}$, and the temporal lag $\gamma_{xy}$ between $X$ and $Y$.
	The adjustment set for identifying the total effect relative to ($X_{t-\gamma_{xy}}, Y_t$) is $\mathcal{B}_{t-\gamma_{xy}}\cup \cdots\cup \mathcal{B}_{t-\gamma_{xy}-\gamma_{max}}\cup \mathcal{X}$
	such that: 
	\begin{enumerate*}
		\item $\mathcal{B} = Parents(X, \mathcal{G})\backslash\{X\}$;
		\item $\mathcal{X}=\{X_{t-\gamma_{xy}-1}, \cdots, X_{t-\gamma_{xy}-\gamma_{max}}\}$ if there exists a loop on $X$ in $\mathcal{G}$, otherwise $\mathcal{X}=\{\emptyset\}$.
	\end{enumerate*}
\end{definition}

\begin{proposition}
	\label{prop:backdoor}
	Given an ASCGL $\mathcal{G}=(\mathcal{V}, \mathcal{E})$ and a maximal lag $\gamma_{max}$, if $\mathcal{B}_{t-\gamma_{xy}}\cup \cdots\cup \mathcal{B}_{t-\gamma_{xy}-\gamma_{max}}\cup \mathcal{X}$ satisfies Definition~\ref{def:back_door_summarygraph} in $\mathcal{G}$ relative to ($X_{t-\gamma_{xy}}, Y_t$) then $\mathcal{B}_{t-\gamma_{xy}}\cup \cdots\cup \mathcal{B}_{t-\gamma_{xy}-\gamma_{max}}\cup \mathcal{X}$ blocks all activated paths between $X_{t-\gamma_{xy}}$ and $Y_t$ going into $X_{t-\gamma_{xy}}$ in every window causal graph associated with $\mathcal{G}$.
\end{proposition}

\begin{sproof}
	If $X$ and $Y$ have no loops in $\mathcal{G}$, then $\mathcal{B}_{t-\gamma_{xy}}\cup \cdots\cup \mathcal{B}_{t-\gamma_{xy}-\gamma_{max}}$ is sufficient to block all paths between $X_{t-\gamma_{xy}}$ and $Y_t$ going into $X_{t-\gamma_{xy}}$ since  all possible parents of $X_{t-\gamma_{xy}}$ are in $\mathcal{B}_{t-\gamma_{xy}}\cup \cdots\cup \mathcal{B}_{t-\gamma_{xy}-\gamma_{max}}$. Given that $\mathcal{G}$ is acyclic, $\mathcal{B}_{t-\gamma_{xy}}\cup \cdots\cup \mathcal{B}_{t-\gamma_{xy}-\gamma_{max}}$ cannot block any directed path nor create any new activated path between $X_{t-\gamma_{xy}}$ and $Y_t$ as there cannot be any descendant of $X_{t-\gamma_{xy}}$ in $\mathcal{B}_{t-\gamma_{xy}}\cup \cdots\cup \mathcal{B}_{t-\gamma_{xy}-\gamma_{max}}$. 
	If $X$ and $Y$ have loops then adjusting on $\mathcal{B}$ up to $\gamma_{max}$ cannot block all back-door paths because there will always be an activated path passing by $\mathcal{B}_{t-\gamma_{xy}-\gamma_{max}-i}$ such that $i>0$. The only way to block this path is to add the past of $X_{t-\gamma_{xy}}$ to the adjustment set.
\end{sproof}

When there exist directed path between $X$ and $Y$ the total effect would no longer be reliable to detect external interventions. For example, in Figure~\ref{fig:linked_anomalous_graphs}, if there's an external intervention on $B$ then $TE^{N}_{C_{t-\gamma_{cd}}\rightarrow D_t} \ne TE^{\bar{N}}_{C_{t-\gamma_{cd}}\rightarrow D_t}$ due to the change in the causal mechanism of the mediator $B$ of $C$ and $D$.
To avoid such cases, for each $X\rightarrow Y$ in a given linked anomalous graph, we need to estimate the direct effect \citep{Pearl_2000} 
of $X$ on $Y$ in the normal regime, defined as
\begin{align}
DE^{{N}}_{X_{t-\gamma_{xy}}\rightarrow Y_t} =& E_{N}[Y_t|do(X_{t-\gamma_{xy}}=x, \mathcal{W}=\textit{w})]  \\
&- E_{N}[Y_t|do(X_{t-\gamma_{xy}}=x', \mathcal{W}=\textit{w})]\nonumber
\end{align}
and the direct effect 
of $X$ on $Y$ in the anomalous regime, defined as
\begin{align}
DE^{\bar{N}}_{X_{t-\gamma_{xy}}\rightarrow Y_t} =& E_{\bar{N}}[Y_t|do(X_{t-\gamma_{xy}}=x, \mathcal{W}=\textit{w})]  \\
&- E_{\bar{N}}[Y_t|do(X_{t-\gamma_{xy}}=x', \mathcal{W}=\textit{w})], \nonumber
\end{align}
where $\mathcal{W}= \mathcal{V}_{t-\gamma_{max}}\cup\cdots\cup\mathcal{V}_{t-\gamma_{xy}}\backslash\{X_{t-\gamma_{xy}}\}\cup\cdots\cup\mathcal{V}_{t}\backslash\{Y_{t}\}$.

Assuming linearity, the \textit{do} from the direct effect expression can be eliminated (i.e., direct effect can be identifyed) using any adjustment set given by the single-door criterion \citep{Pearl_2000} which is not applicable in ASCGLs. In the following, we present an adjustment set for direct effects in ASCGLs:

\begin{definition}[An adjustment set for direct effects in an ASCGL]
	\label{def:single_door_summarygraph}
	Consider an ASCGL $\mathcal{G}=(\mathcal{V}, \mathcal{E})$, a maximal lag $\gamma_{max}$, two vertices $X$ and $Y$ such that $X\rightarrow Y$ in $\mathcal{G}$, and the temporal lag $\gamma_{xy}$ between $X$ and $Y$.
	An adjustment set for identifying the direct effect relative to ($X_{t-\gamma_{xy}}, Y_t$) is $\mathcal{B}_{t}\cup \cdots\cup \mathcal{B}_{t-\gamma_{max}}\cup \mathcal{X}\cup \mathcal{Y}$
	such that:
	\begin{enumerate*} 
		\item $\mathcal{B}= Parents(Y,\mathcal{G})\backslash\{X, Y\}$;
		\item $\mathcal{X}=\{X_t, \cdots, X_{t-\gamma_{max}}\}\backslash\{X_{t-\gamma_{xy}}\}$ and $\mathcal{Y}=\{Y_{t-1}, \cdots, Y_{t-\gamma_{max}}\}$ if there exists a loop on $Y$ in $\mathcal{G}$, otherwise $\mathcal{Y}=\{\emptyset\}$.
	\end{enumerate*}
\end{definition}

\begin{proposition}
	\label{prop:singledoor}
	Given an ASCGL $\mathcal{G}=(\mathcal{V}, \mathcal{E})$ and a maximal lag $\gamma_{max}$, 
	if $\mathcal{B}_{t}\cup \cdots\cup \mathcal{B}_{t-\gamma_{max}}\cup \mathcal{X}\cup \mathcal{Y}$ satisfies Definition~\ref{def:single_door_summarygraph} in $\mathcal{G}$ relative to ($X_{t-\gamma_{xy}}, Y_t$) then $\mathcal{B}_{t}\cup \cdots\cup \mathcal{B}_{t-\gamma_{max}}\cup \mathcal{X}\cup \mathcal{Y}$ blocks all activated paths between $X_{t-\gamma_{xy}}$ and $Y_t$ in every window causal graph associated with $\mathcal{G}$ except the direct path $X_{t-\gamma_{xy}}\rightarrow Y_t$.
\end{proposition}

\begin{sproof}
	If $X$ and $Y$ have no loops in $\mathcal{G}$, then $\mathcal{B}_{t}\cup \cdots\cup \mathcal{B}_{t-\gamma_{max}}$ is sufficient to block all paths from $X_{t-\gamma_{xy}}$ to $Y_t$ except $X_{t-\gamma_{xy}}\rightarrow Y_t$ since and all possible parents of $Y_t$ are in $\mathcal{B}_{t}\cup \cdots\cup \mathcal{B}_{t-\gamma_{max}}$. Given that $\mathcal{G}$ is acyclic $\mathcal{B}_{t}\cup \cdots\cup \mathcal{B}_{t-\gamma_{max}}$ cannot create any new activated path that is not blocked by $\mathcal{B}_{t}\cup \cdots\cup \mathcal{B}_{t-\gamma_{max}}$ since no vertex in $\mathcal{B}_{t}\cup \cdots\cup \mathcal{B}_{t-\gamma_{max}}$ is a  descendant of $Y_t$. If $X$ and $Y$ have loops then adjusting on $\mathcal{B}$ up to $\gamma_{max}$ cannot block all paths because there can be an activated path passing by the future of $X_{t-\gamma_{xy}}$ or by the past of $Y_{t}$. The only way to block these types of paths is to add  $\mathcal{X}\cup \mathcal{Y}$ to the adjustment set.
\end{sproof}

\subsection{An algorithm for root cause identification}

Here, we describe our main method called EasyRCA\footnote{Code available at https://github.com/ckassaad/\\EasyRCA}, in which the pseudocode is provided in Algorithm~\ref{algo:EasyRCA}. The algorithm starts by finding linked anomalous graphs (line 1). 
Then for each linked anomalous graph, it searches for the sub-roots and time-defying vertices (line 3).
Finally, it searches for the rest of the root causes by comparing direct effects in the normal regime with direct effects in the anomalous regime (lines 4-16).
The conditions in lines 12 and 14 need the minimality condition because if $X$ and $Y$ are statistically independent given the adjustment set in the normal regime then an intervention on $Y$ might not imply any change to the statistical distribution thus one cannot conclude on the presence of interventions.
The for-loop in line 2 can be parallelized since as showed in Proposition~\ref{prop:disjoint_set_of_link_anomalous_graphs} and \ref{prop:modularity_of_link_anomalous_graphs}, linked anomalous graphs are modular. 

\begin{theorem}
	Given an ASCGL $\mathcal{G}=(\mathcal{V}, \mathcal{E})$, a set anomalous vertices $\mathcal{A} \subseteq \mathcal{V}$, the distribution of the time series in the normal regime $\mathcal{{N}}$ and in the anomalous regime $\mathcal{\bar{N}}$, and the maximal lag between a cause and an effect $\gamma_{max}$, under Assumption~\ref{ass:anomaly_propagation} and the minimality condition, EasyRCA is capable of identifying the set of root cause $\mathcal{C}$ of $\mathcal{A}$.
\end{theorem}
\begin{sproof}
	It follows from Propositions~\ref{prop:disjoint_set_of_link_anomalous_graphs},~ \ref{prop:modularity_of_link_anomalous_graphs},~\ref{prop:roots_sub_root_time_defying}, \ref{prop:singledoor}.
\end{sproof}

Note that we can also distinguish between parametric and structural interventions. Given that $DE^{N}_{X_{t-\gamma_{xy}}\rightarrow Y_t} \ne DE^{\bar{N}}_{X_{t-\gamma_{xy}}\rightarrow Y_t}$, if $DE^{\bar{N}}_{X_{t-\gamma_{xy}}\rightarrow Y_t}=0$, we conclude that the intervention on $Y$ is structural, otherwise, we conclude that it is parametric.

\begin{algorithm}[h!]
	\caption{EasyRCA}
	\label{algo:EasyRCA}
	\begin{algorithmic}[1]
		\REQUIRE ASCGL $\mathcal{G}=(\mathcal{V}, \mathcal{E})$, distribution of the time series in the normal regime $\mathcal{{N}}$ and in the anomalous regime $\mathcal{\bar{N}}$, maximal lag $\gamma_{max}$, Anomalies $\mathcal{A}$
		\STATE $\mathcal{L}^1, \cdots, \mathcal{L}^m =$ list of linked anomalous graphs as in Definition~\ref{def:LAG}
		\FOR{$i \in \{1,\cdots, m\}$}
		\STATE Identify sub-root vertices $\mathcal{R}^i$ and time defying vertices $\mathcal{T}^i$ using Definition~\ref{def:sub_root} and \ref{def:time_defying}
		\STATE Let $\mathcal{D}^i=[]$
		\STATE Let $\mathcal{A}^i$ be the set of vertices in $\mathcal{L}^i$
		\FOR{$Y$ in $\mathcal{A}^i\backslash\{\mathcal{R}^i\cup  \mathcal{T}^i\}$} 
		\FOR{$X$ in $Parents(Y, \mathcal{G})$}
		\STATE $\bar\gamma_{xy}$: anomaly lag between $X$ and $Y$ 		
		\FOR{$\gamma_{xy}$ in $\{\bar\gamma_{xy}, \cdots, \gamma_{max}\}$}
		\STATE Identify 
		$\mathcal{B}_{t}\cup \cdots\cup \mathcal{B}_{t-\gamma_{max}}\cup \mathcal{X}\cup \mathcal{Y}$ using Definition~\ref{def:single_door_summarygraph}
		\STATE Estimate $DE^{N}_{X_{t-\gamma_{xy}}\rightarrow Y_t}$
		\IF{$DE^{N}_{X_{t-\gamma_{xy}}\rightarrow Y_t}\ne 0$} 
		\STATE Estimate $DE^{\bar{N}}_{X_{t-\gamma_{xy}}\rightarrow Y_t}$
		\IF{$DE^{N}_{X_{t-\gamma_{xy}}\rightarrow Y_t} \ne DE^{\bar{N}}_{X_{t-\gamma_{xy}}\rightarrow Y_t}$}
		\STATE $\mathcal{D}^i = [\mathcal{D}^i, Y]$
		\STATE Break
		\ENDIF 
		\ENDIF 
		\ENDFOR 
		\ENDFOR 	
		\ENDFOR 	
		\ENDFOR
		\STATE \textbf{Return} $\mathcal{R}$, $\mathcal{T}$, $\mathcal{D}$ 
	\end{algorithmic}
\end{algorithm}

\section{EXPERIMENTS}\label{sec:exp}

We propose first an extensive analysis on simulated data,
generated from random causal graphs; then we perform an
analysis on a real word dataset.

\subsection{Experimental Setup}


In practice, to test if $DE^{N}_{X_{t-\gamma_{xy}}\rightarrow Y_t} \ne DE^{\bar{N}}_{X_{t-\gamma_{xy}}\rightarrow Y_t}$,  we fit 11 multiple linear regressions: 
\begin{equation*}
Y_t = \hat{a}_x X_{t-\gamma_{xy}} + \sum_{B_{t-\gamma_{by}}\in \mathcal{B}_{t}\cup \cdots\cup \mathcal{B}_{t-\gamma_{max}}\cup \mathcal{X}\cup \mathcal{Y}} \hat{a}_b B_{t-\gamma_{by}}  + \epsilon^y_t, 
\end{equation*}
such that $\mathcal{B}_{t}\cup \cdots\cup \mathcal{B}_{t-\gamma_{max}}\cup \mathcal{X}\cup \mathcal{Y}$ is identified using Definition~\ref{def:single_door_summarygraph}.
One of them fitted on the anomalous data and 10 fitted on different chunks of the normal data.
Then using the Grubbs-test \citep{Grubbs_1950} we check if the coefficient $\hat{a}_x$ of the anomalous data is significantly different than the 10 others $\hat{a}_x$ from the normal data.
To test if $TE^{\bar{N}}_{X_{t-\gamma_{xy}}\rightarrow Y_t} = 0$,  we use a t-test on the coefficient $\hat{a}_x$ of the anomalous data.

\textbf{Baselines}:
We compare EasyRCA with three other methods\footnote{We implemented CloudRanger and MicroCause and we adapted WhyMDC based on the DoWhy package.}: CloudRanger, MicroCause and a naive adaptation of WhyMDC to time series for which we provide a window causal graph. 
Since CloudRanger and MicroCause try to solve a harder problem compared to EasyRCA by inferring the summary causal graph from anomalous data (while EasyRCA considers that the summary causal graph is given), we also consider another version of EasyRCA, denoted as EasyRCA$^*$, where we suppose that the summary causal graph is not given. In the first step of EasyRCA$^*$, we infer the window causal graph from normal data using PCMCI \citep{Runge_2020}, the same causal discovery algorithm used by MicroCause, and then we deduce the summary causal graph from it. Note that the summary causal graph obtained from the window graph that is inferred by PCMCI can be cyclic even if the true summary causal graph is acyclic (due to estimation errors). In such cases, we consider that EasyRCA$^*$ does not identify any root cause.

\textbf{Hyper-parameters}:
For EasyRCA, EasyRCA$^*$ and MicroCause, we set the maximal lag $\gamma_{max}$ to $3$ and for all methods (even though the true $\gamma_{max}$ is smaller in our simulation study), we set the significance threshold to $0.01$. 
For EasyRCA$^*$, CloudRanger and MicroCause we use a Fisher-z-test \citep{Kalisch_2007}, which is commonly used for causal discovery when linearity and gaussianity are satisfied. 
Furthermore, for CloudRanger and MicroCause we set the walk length to $1000$ and the backward step threshold to $0.1$.
Lastly, other hyper-parameters in WhyMDC were set to the default values in the DoWhy package.

\textbf{Evaluation}:
To assess the quality of identifying root causes, we use the F1-score.
Since by construction EasyRCA identifies sub-root and time defying vertices as root causes we do not evaluate their detection. This ensures a fair comparison with other methods.

\subsection{Simulated Data}

For simulated data, we start by randomly generating $30$ different ASCGLs such that each graph contains $6$ vertices, has a maximal degree between $4$ and $5$, and has one root vertex.
We consider that all lags in the window causal graph associated with any of the generated ASCGL are equal to $1$. So the generative process (the SCM)  is the following:
\begin{equation*}
Y_t = \sum_{X_{t-1}\in Parents(Y_t, \mathcal{G}_w)} aX_{t-1} + 0.1\xi^y_t
\end{equation*}
where $a\sim U\{0.1, 1\}$,  $\xi^y_t \sim \mathcal{N}(0, 1)$,
$Y_t$ denotes the value of the vertex at time $t$, $Parents(Y_t)$ denotes the direct parents of $Y_t$ in the window causal graph. 

For each ASCGL, we choose two root causes (two vertices that will undergo an intervention): the first is the root of the ASCGL and the second is a randomly chosen vertex among the non-root vertices. 
We propagate the effect of each intervention according to the generating process toward all the descendants of the vertex which underwent an intervention.
We set the starting time of each intervention according to the generative process. For example, if the root vertex $X$ has an intervention at time $t$ and a vertex $Y$ randomly selected to undergo an intervention is a child of $X$ then the starting time of the intervention on $Y$ is $t+1$.
In general, the starting time of the intervention does not need to respect the generative process, but we chose to respect it to avoid any time-defying vertices which would give an advantage to our method.

In our experiments, we consider the two types of interventions separately and we vary the anomaly size between $100$ and $2000$. 
In the case of structural interventions, values of the root vertex and values of a randomly chosen non-root vertex in the anomalous interval are replaced by data drawn form the distribution $Exp(2)$.
In the case of parametric interventions, values of the root are set similarly to structural interventions. Then values of the non-root vertex are regenerated with new coefficients from $U(0.1, 1)$.

\textbf{Results}:
In Figure~\ref{fig:res_sim_1}, we report the performance of each method at detecting structural interventions 
with respect to the anomaly size. As one can see, EasyRCA and EasyRCA$^*$ clearly outperform other methods, and their performances increases (and their variance decreases) significantly between the anomaly of size $100$ and the anomaly of size $2000$ reaching an F1-score of $1$ for both EasyRCA and EasyRCA$^*$. The small difference in the performance of EasyRCA and EasyRCA$^*$ shows that our method is robust with respect to small errors in the ASCGL.
WhyMDC and MicroCause have similar results and outperform CloudRanger. 
In Figure~\ref{fig:res_sim_2}, we report the performance of each method at detecting parametric interventions 
with respect to the anomaly size. As before, EasyRCA and EasyRCA$^*$ outperforms other methods but now the difference between EasyRCA and EasyRCA$^*$ is more visible 
and all other methods suffer. However, it is worth noting, that unlike CloudRanger, WhyMDC and MicroCause were able to detect the root vertex of each graph as a root cause.

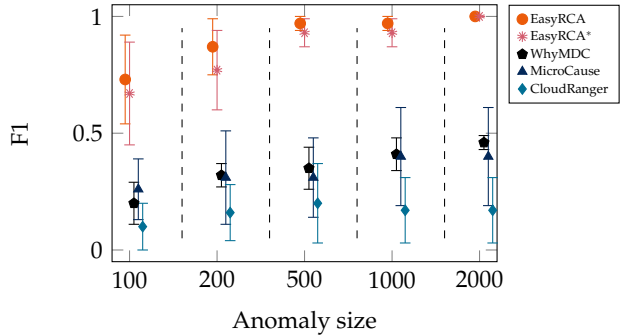
\begin{figure}[htb]
	\centering
	\begin{tikzpicture}[font=\small]
	\renewcommand{\axisdefaulttryminticks}{4}
	\pgfplotsset{every major grid/.append style={densely dashed}}
	\pgfplotsset{every axis legend/.append style={cells={anchor=west},fill=white, at={(0.02,0.98)}, anchor=north west}}
	\begin{axis}[
	log ticks with fixed point,
	xmin = 0.8,
	xmax = 5.2,
	xtick = {1,2, 3, 4, 5},
	xticklabels = {$100$, $200$, $500$, $1000$, $2000$},
	ymin=-0.05,
	ymax=1.05,
	grid=minor,
	scaled ticks=true,
	xlabel = {Anomaly size},
	ylabel = {F1},
	height = 5.0cm,
	width=6.7cm,
	legend style={nodes={scale=0.55, transform shape}},
	legend pos=outer north east
	]
	\addplot[ugaorange,only marks,mark=*, error bars/.cd, y dir=both,y explicit] plot coordinates{
		(0.95, 0.73) +- (0.19, 0.19)
		(1.95, 0.87) +- (0.12, 0.12)
		(2.95, 0.97) +- (0.03, 0.03)
		(3.95, 0.97) +- (0.03, 0.03)
		(4.95, 1) +- (0.0, 0.0)
	};
	\addplot[easyred,only marks,mark=10-pointed star, error bars/.cd, y dir=both,y explicit] plot coordinates{
		(1, 0.67) +- (0.22, 0.22)
		(2, 0.77) +- (0.17, 0.17)
		(3, 0.93) +- (0.06, 0.06)
		(4, 0.93) +- (0.06, 0.06)
		(5, 1) +- (0.0, 0.0)
	};
	\addplot[black,only marks,mark=*, mark=pentagon*, error bars/.cd, y dir=both,y explicit] plot coordinates{
		(1.05, 0.2) +- (0.09, 0.09)
		(2.05, 0.32) +- (0.05, 0.05)
		(3.05, 0.35) +- (0.09, 0.09)
		(4.05, 0.41) +- (0.07, 0.07)
		(5.05, 0.46) +- (0.03, 0.03)
	};
	\addplot[easyblue,only marks, mark=triangle*, error bars/.cd, y dir=both,y explicit] plot coordinates{
		(1.1, 0.26) +- (0.13, 0.13)
		(2.1, 0.31) +- (0.2, 0.2)
		(3.1, 0.31) +- (0.17, 0.17)
		(4.1, 0.4) +- (0.21, 0.21)
		(5.1, 0.4) +- (0.21, 0.21)
	};
	\addplot[easycyan,only marks,mark=diamond*, error bars/.cd, y dir=both,y explicit] plot coordinates{
		(1.15, 0.1) +- (0.1, 0.1)
		(2.15, 0.16) +- (0.12, 0.12)
		(3.15, 0.2) +- (0.17, 0.17)
		(4.15, 0.17) +- (0.14, 0.14)
		(5.15, 0.17) +- (0.14, 0.14)
	};
	\draw [dashed] (80,10) -- (80,100);
	\draw [dashed] (180,10) -- (180,100);
	\draw [dashed] (280,10) -- (280,100);
	\draw [dashed] (380,10) -- (380,100);
	\legend{{EasyRCA}, {EasyRCA$^*$}, {WhyMDC}, {MicroCause}, {CloudRanger}}
	\end{axis}
	\end{tikzpicture}
	\caption{Mean and variance of F1-scores with respect to structural interventions over 30 graphs containing one linked anomalous graph with one sub-root vertex and one structural intervention.}
	\label{fig:res_sim_1}
\end{figure}

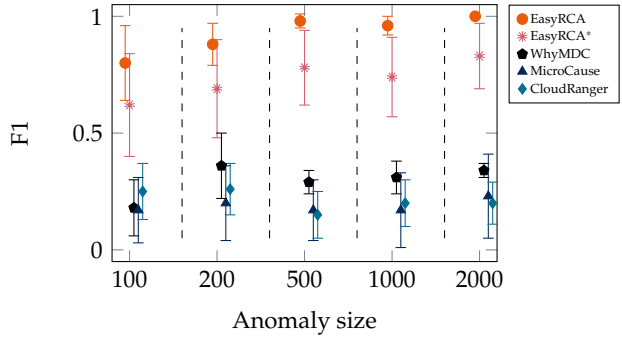
\begin{figure}[htb]
	\centering
	\begin{tikzpicture}[font=\small]
	\renewcommand{\axisdefaulttryminticks}{4}
	\pgfplotsset{every major grid/.append style={densely dashed}}
	\pgfplotsset{every axis legend/.append style={cells={anchor=west},fill=white, at={(0.02,0.98)}, anchor=north west}}
	\begin{axis}[
	log ticks with fixed point,
	xmin = 0.8,
	xmax = 5.2,	
	xtick = {1,2, 3, 4, 5}, 
	xticklabels = {$100$, $200$, $500$, $1000$, $2000$},
	ymin=-0.05,
	ymax=1.05,
	grid=minor,
	scaled ticks=true,
	xlabel = {Anomaly size},
	ylabel = {F1},
	height = 5.0cm,
	width=6.7cm,
	legend style={nodes={scale=0.55, transform shape}},
	legend pos=outer north east
	]
	\addplot[ugaorange,only marks,mark=*, error bars/.cd, y dir=both,y explicit] plot coordinates{
		(0.95, 0.8) +- (0.16, 0.16)
		(1.95, 0.88) +- (0.09, 0.09)
		(2.95, 0.98) +- (0.03, 0.03)
		(3.95, 0.96) +- (0.04, 0.04)
		(4.95, 1) +- (0.0, 0.0)
	};
	\addplot[easyred,only marks,mark=10-pointed star, error bars/.cd, y dir=both,y explicit] plot coordinates{
		(1, 0.62) +- (0.22, 0.22)
		(2, 0.69) +- (0.21, 0.21)
		(3, 0.78) +- (0.16, 0.16)
		(4, 0.74) +- (0.17, 0.17)
		(5, 0.83) +- (0.14, 0.14)
	};
	\addplot[black,only marks,mark=*, mark=pentagon*, error bars/.cd, y dir=both,y explicit] plot coordinates{
		(1.05, 0.18) +- (0.12, 0.12)
		(2.05, 0.36) +- (0.14, 0.14)
		(3.05, 0.29) +- (0.05, 0.05)
		(4.05, 0.31) +- (0.07, 0.07)
		(5.05, 0.34) +- (0.03, 0.03)
	};
	\addplot[easyblue,only marks, mark=triangle*, error bars/.cd, y dir=both,y explicit] plot coordinates{
		(1.1, 0.17) +- (0.14, 0.14)
		(2.1, 0.2) +- (0.16, 0.16)
		(3.1, 0.17) +- (0.13, 0.13)
		(4.1, 0.17) +- (0.16, 0.16)
		(5.1, 0.23) +- (0.18, 0.18)
	};
	
	\addplot[easycyan,only marks,mark=diamond*, error bars/.cd, y dir=both,y explicit] plot coordinates{
		(1.15, 0.25) +- (0.12, 0.12)
		(2.15, 0.26) +- (0.11, 0.11)
		(3.15, 0.15) +- (0.1, 0.1)
		(4.15, 0.2) +- (0.1, 0.1)
		(5.15, 0.2) +- (0.09, 0.09)
	};
	\draw [dashed] (80,10) -- (80,100);
	\draw [dashed] (180,10) -- (180,100);
	\draw [dashed] (280,10) -- (280,100);
	\draw [dashed] (380,10) -- (380,100);
	\legend{{EasyRCA}, {EasyRCA$^*$}, {WhyMDC}, {MicroCause}, {CloudRanger}}
	\end{axis}
	\end{tikzpicture}
	\caption{Mean and variance of F1-scores with respect to parametric interventions over 30 graphs containing one linked anomalous graph with one sub-root vertex and one parametric intervention.}
	\label{fig:res_sim_2}
\end{figure}

\subsection{Real Data}

For real data, we consider a dataset\footnote{The real IT monitoring data is available at
	\url{https://easyvista2015-my.sharepoint.com/personal/aait-bachir_easyvista_com/_layouts/15/onedrive.aspx?id=\%2Fpersonal\%2Faait\%2Dbachir\%5Feasyvista\%5Fcom\%2FDocuments\%2FLab\%2FPublicData&ga=1}
}
which consists of eight time series collected from an IT monitoring system with a one-minute sampling rate provided by EasyVista\footnote{https://www.easyvista.com/fr/produits/ev-observe} such that each of these time series is considered anomalous and all collective anomalies are considered to have the same time of appearance and of size $100$.
The corresponding ASCGL is provided in Figure~\ref{fig:storm_graph} 
where PMDB 
represents the extraction of some information about the messages received by the Storm ingestion system;
MDB 
refers to an activity of a process that orient messages to other process with respect to different types of messages;
CMB 
represents the activity of extraction of metrics from messages;
MB 
represents the activity of insertion of data in a database;
LMB 
reflects the updates the last values of metrics in Cassandra;
RTMB 
represents the activity of searching to merge of data with information coming from the check message bolt;
GSIB 
represents the activity of insertion of historical status in database.
ESB 
represents the activity of writing data in Elasticsearch. 
According to EasyVista's system experts, PMDB and ESB are expected to be the root causes of these anomalies.

EasyRCA inferred $3$ roots causes, PMDB as a root vertex, in addition to RTMB and ESB as structural interventions. 
EasyRCA$^*$ inferred $5$ roots causes, PMDB, GSIB and MB as root vertices, in addition to RTMB and ESB as structural interventions.
MicroCause inferred that PMDB and MB are the root causes of the anomalies.
CloudRanger inferred that GSIB and MDB are the root causes.
We did not apply WhyMDC in this real world application because the true window causal graph is unknown. 
In terms of the trade-off between false positives and false negatives, EasyRCA gives the best result.

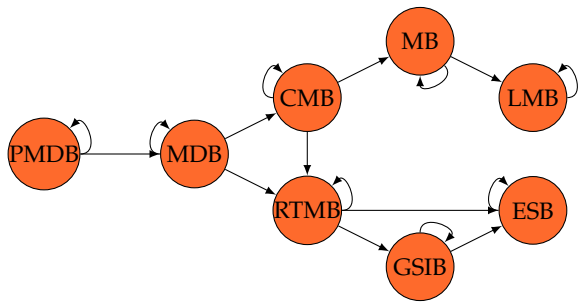
\begin{figure}
	\centering
	\begin{tikzpicture}[{black, circle, draw, inner sep=0}]
	\tikzset{nodes={draw,rounded corners},minimum height=0.9cm,minimum width=0.9cm, font=\footnotesize}	
	\tikzset{anomalous/.append style={fill=easyorange}}
	\tikzset{rc/.append style={fill=easyorange}}
	
	\node[anomalous] (PMDB) at (0,0) {PMDB} ;
	\node[anomalous] (MDB) at (2,0) {MDB};
	\node[anomalous] (CMB) at (3.5,0.75) {CMB};
	\node[anomalous] (MB) at (5,1.5) {MB};
	\node[anomalous] (LMB) at (6.5,0.75) {LMB};
	\node[anomalous] (RTMB) at (3.5,-0.75) {RTMB} ;
	\node[anomalous] (GSIB) at (5,-1.5) {GSIB};
	\node[anomalous] (ESB) at (6.5,-0.75) {ESB} ;
	
	\draw[->,>=latex] (PMDB) -- (MDB);
	\draw[->,>=latex] (MDB) -- (CMB);
	\draw[->,>=latex] (CMB) -- (MB);
	\draw[->,>=latex] (CMB) -- (RTMB);
	\draw[->,>=latex] (MB) -- (LMB);
	\draw[->,>=latex] (MDB) -- (RTMB);
	\draw[->,>=latex] (RTMB) -- (GSIB);
	\draw[->,>=latex] (RTMB) -- (ESB);
	\draw[->,>=latex] (GSIB) -- (ESB);
	
	\draw[->,>=latex] (PMDB) to [out=0,in=45, looseness=2] (PMDB);
	\draw[->,>=latex] (MDB) to [out=180,in=135, looseness=2] (MDB);
	\draw[->,>=latex] (CMB) to [out=180,in=135, looseness=2] (CMB);
	\draw[->,>=latex] (MB) to [out=-45,in=-90, looseness=2] (MB);
	\draw[->,>=latex] (LMB) to [out=0,in=45, looseness=2] (LMB);
	\draw[->,>=latex] (RTMB) to [out=0,in=45, looseness=2] (RTMB);
	\draw[->,>=latex] (GSIB) to [out=90,in=45, looseness=2] (GSIB);
	\draw[->,>=latex] (ESB) to [out=180,in=135, looseness=2] (ESB);
	
	\end{tikzpicture}
	\caption{ASCGL of the normal regime of an IT monitoring system. All vertices are anomalous in the anomalous regime. According to EasyVista's system experts, PMDB and ESB are expected to be the root causes of these anomalies.}
	\label{fig:storm_graph}
\end{figure}%

\section{Conclusion}
\label{sec:conc}
We adressed the problem of identifying root causes of collective anomalies using observational time series and an ASCGL of the normal regime of a given system. We showed that the problem can be divided into many independent subproblems and that all root causes can be identified using the graph and the data.
For future work, it would be interesting to extend this method for cyclic summary causal graphs, for nonlinear SCMs and to allow for hidden common causes.

\subsubsection*{Acknowledgements}
We thank Ali Aït-Bachir, Christophe de Bignicourt and Hosein Mohanna from EasyVista for providing the IT monitoring data along with the underlying causal graph and for localizing anomalies in the data.

\bibliographystyle{plainnat}
\bibliography{references.bib}

\end{document}